\documentclass[letterpaper, 10 pt, journal, twoside]{IEEEtran}
\usepackage{cite}
\usepackage{enumerate}
\usepackage[english]{babel}
\usepackage[utf8]{inputenc}
\usepackage{algorithm}
\usepackage[noend]{algpseudocode}
\usepackage{amsmath}
\usepackage{amssymb}
\usepackage{graphicx}
\usepackage{hyperref}
\usepackage{graphics}
\newtheorem{definition}{Definition}
\newtheorem{lemma}{Lemma}
\newtheorem{theorem}{Theorem}
\newtheorem{corollary}{Corollary}
\newtheorem{proposition}{Proposition}
\newtheorem{remark}{Remark}
\IEEEoverridecommandlockouts 
\newcommand{\R}{\mathbb{R}}
\newcommand{\e}{\epsilon}
\newcommand{\g}{\nabla}
\newcommand{\E}{\mathbb{E}}
\newcommand{\F}{\mathcal{F}}

\usepackage{tikz}
\usepackage{textcomp}
\usepackage{hyperref}
\usepackage{lipsum}

\newcommand\copyrighttext{%
  \footnotesize \textcopyright 2019 IEEE. Personal use is permitted, but republication/redistribution requires IEEE permission. This paper is accepted at IEEE Control Systems Letters
  DOI: 10.1109/LCSYS.2019.2916467}
\newcommand\copyrightnotice{%
\begin{tikzpicture}[remember picture,overlay]
\node[anchor=south,yshift=10pt] at (current page.south) {\fbox{\parbox{\dimexpr\textwidth-\fboxsep-\fboxrule\relax}{\copyrighttext}}};
\end{tikzpicture}%
}

\title{
An Online Sample Based Method for Mode Estimation using ODE Analysis of Stochastic Approximation Algorithms}

\author{Chandramouli Kamanchi$^{1}$ Raghuram Bharadwaj Diddigi$^{2}$ Prabuchandran K.J.$^{3}$ Shalabh Bhatnagar$^{4,5}$
\thanks{$^{1}$ Department of Computer Science and Automation, Indian Institute of Science, Bangalore, India.
        {\tt\small chandramouli@iisc.ac.in}}%
\thanks{$^{2}$ Department of Computer Science and Automation, Indian Institute of Science, Bangalore, India.
        {\tt\small raghub@iisc.ac.in}}%
\thanks{$^{3}$ Amazon-IISc Postdoctoral fellow, Indian Institute of Science, Bangalore, India.
        {\tt\small prabuchandra@iisc.ac.in}}
\thanks{$^{4}$ Department of Computer Science and Automation and Department of Robert Bosch Centre for Cyber-Physical Systems, Indian Institute of Science, Bangalore, India.
        {\tt\small shalabh@iisc.ac.in}}%
\thanks{$^{5}$ Supported by RBCCPS, IISc and a grant from the Department of Science and Technology, India.}}

\begin{document}
\pagestyle{empty} 


\maketitle
\thispagestyle{empty}
\copyrightnotice

\begin{abstract}
One of the popular measures of central tendency that provides better representation and interesting insights of the data compared to the other measures like mean and median is the metric mode. 
If the analytical form of the density function is known, mode is an argument of the maximum value of the density function and one can
apply optimization techniques to find the mode. In many of the practical applications, the analytical form of the density is not known and only the samples from the distribution are available. Most of the techniques proposed in the literature for estimating the mode from the samples  assume that all the samples are available beforehand. Moreover, some of the techniques employ computationally expensive operations like sorting. In this work we provide a computationally effective, on-line iterative algorithm that estimates the mode of a unimodal smooth density given only the samples generated from the density. Asymptotic convergence of the proposed algorithm using an ordinary differential equation (ODE) based analysis is provided. We also prove the stability of estimates by utilizing the concept of regularization. Experimental results further demonstrate the effectiveness of the proposed algorithm.  
\end{abstract}

\section{Introduction}

\IEEEPARstart{T}{here} are many metrics that are used to represent the central tendency of the data. Among them, the popular ones are mean, median and mode. Mean is extensively studied due to its simplicity, linearity and ease of estimation via sample averages. However mean is susceptible to outliers. For example, when we are estimating the mean from a finite number of samples, one bad outlier can shift the estimate far away from the original mean. Also, in some of the applications, mean may not be the desired quantity to analyze. For example, it is often interesting to know the income that majority of the population in a country earns rather than the average income of the country which is a skewed quantity. 


In this work, we focus on finding the mode of a density when the analytical form of it is not known. That is, we are given only the samples of the distribution and we need to estimate the mode from these samples. We utilize stochastic approximation techniques to solve this problem. Stochastic approximation is a popular paradigm that is applied sucessfully to analyze random iterative models \cite{jaakkola1994convergence,zhou2017stochastic,zhou2017mirror}.

We first discuss some of the works that have been reported in the literature for estimating the mode. 
This problem of estimation of the mode of a unimodal density has been first considered in \cite{parzen1962estimation} where a sequence of density functions is iteratively constructed from the samples and the respective modes are calculated as maximum likelihood estimates. It is shown that these estimates of mode converge in probability to the actual mode.

We can broadly classify the solution techniques for the mode estimation problem into two groups. The first group comprises a non-parametric way of estimating the mode where the mode is estimated directly from the sample data without constructing the density function. The second group of methods comprises the parametric way of estimation in which the density function is approximately constructed and the mode is computed using optimization techniques. 

 In \cite{grenander1965some}, a non-parametric estimator (popularly known as Grenander's estimate) for estimating the mode directly from the samples is proposed. The later developments of mode estimation methods are based on the idea that the mode is situated at the center of an interval of selected length that contains majority of observed points. A sequence of nested intervals and intermediate mode estimates is constructed based on the above described idea and mode is taken to be the point that these intermediate estimates of mode converge to. Different ways of selecting the interval lengths are studied in \cite{chernoff1964estimation,wegman1971note} along with their convergence properties. A variant of this idea involves selecting the interval of shortest length that contains some desired number of points instead of deciding the lengths of interval. The estimation methods in \cite{dalenius1965mode,venter1967estimation,robertson1974iterative} are based on this idea. Detailed survey of the above discussed techniques along with their robustness is extensively studied in \cite{bickel2002robust, bickel2006fast}.

 In \cite{bickel2003robust}, a parametric method of estimating the mode is proposed. The idea here is to fit the samples to a normal distribution. Then the mode is estimated by calculating the mean of this fitted normal distribution. This idea is recently extended to find the multivariate mode in \cite{hsu2013efficient}. In \cite{griffin2005multiscale}, multivariate mean shift algorithm is proposed for estimating the multivariate mode from the samples. 
The idea here is to iteratively shift the estimated mode towards the actual mode using Gaussian kernels. In \cite{kirschstein2016minimum}, a minimum volume peeling method is proposed to estimate the multivariate mode from the sample data. The idea here is to iteratively construct subsets of the set of samples with minimum volume and discard the remaining points. The mode is then calculated by averaging the points in the constructed subset. This is based on the observation that mode is generally situated in the minimum volume set of a given fixed number of samples. An effective way of selecting the subset of points is discussed in \cite{kirschstein2016minimum}.


Most of the algorithms considered in the literature so far make the assumption that all the samples are available upfront. These techniques cannot be extended to the case of streaming data where the samples arrive online one at a time. Also, the non-parametric techniques (refer \cite{bickel2006fast}) require the samples to be in a sorted order. 

Our proposed algorithm is fundamentally different from the above algorithms in the sense that ours is an online algorithm that works with the data as it becomes available. This enables us to work with online samples without storing them in the memory. Also, we do not resort to any computationally expensive operations like sorting. In addition, our algorithm works for both univariate and multivariate distributions. We provide a convergence analysis of our proposed technique and show the robustness of our technique using simulation results.

Our work is closest to \cite{tsybakov1990recursive}. In \cite{tsybakov1990recursive} a gradient-like recursive procedure for estimating the mode is proposed and convergence is provided utilizing the theory of martingales. In our work to mitigate the lack of analytical form of density, we construct a kernel based representation (refer section II) of the density function and use stochastic gradient ascent to calculate mode. Our work is different from \cite{tsybakov1990recursive} in the following ways.

\begin{itemize}

\item Our proposed algorithm is based on assumptions different from those of \cite{tsybakov1990recursive}. Moreover, we prove the stability of the mode estimates by introducing the concept of regularization \cite{tikhonov2013numerical}.

\item We demonstrate the effectiveness of our algorithm by providing empirical evaluation on well-known distributions.

\item  Our convergence proof utilizes the well-known ODE based analysis of stochastic approximation algorithms. To the best of our knowledge, ours is the first work that makes use of ODE based analysis in the context of mode estimation.
\end{itemize}




\section{Background and Preliminaries}
To begin, suppose we have a probability space $(\Omega, \mathcal{F}, \mathbb{P})$ and a random vector $X: \Omega \to \mathbb{R}^{p}$ with a smooth density function $f(x).$ A mode of the random vector $X$ is defined as an argument of the maximum of $f.$ Suppose we have a unique mode i.e. $f$ has a unique maximizer then the mode is a measure of central tendency of the distribution of $X.$ A natural problem that arises is the estimation of the mode given a sequence of independent and identically distributed samples of $X$ denoted by $X_1,X_2,\cdots, X_n.$   In what follows we provide necessary formal definitions and prove some properties that are utilized in the motivation and the convergence of our proposed algorithm to solve this problem.
\subsection{Approximation of Identity}
\begin{definition}
The convolution of two functions $u$ and $v$ on $\mathbb{R}^p$ is defined as $$(u*v)(x):=\int_{\mathbb{R}^p}u(x-t)v(t)dt  \ \ \ (x \in \mathbb{R}^p).$$
\end{definition}
\begin{definition}
Given a function $K: \mathbb{R}^p \to \mathbb{R}$ and $\epsilon >0$ we define $$K_{\epsilon}(x):=\epsilon^{-p}K\Big{(}\frac{x}{\epsilon}\Big{)}.$$
For a given function $K$ as above, the family of functions $\{K_{\e} |\e>0\}$ is called approximation of the identity.
\end{definition}

\begin{lemma}
\label{kerlem}
Suppose $\int_{\R^{p}}|K(t)|dt<\infty$. Then, given any $\e>0$ 
\begin{enumerate}
\item $\int_{\R^{p}} K_{\e}=\int_{\R^{p}} K.$ \newline
\item $\int_{||x||>\delta}|K_{\e}| \to 0$ as $\e \to 0$, for any fixed $\delta >0.$
\end{enumerate}
\end{lemma}
\begin{proof}
For statement 1 choose $y=\frac{x}{\e}.$ So $dx= \e^{p} \ dy.$ Then we have 
$$\int_{\R^{p}}K_{\e}=\frac{1}{\e^{p}}\int_{\R^p} K\Big{(}\frac{x}{\e}\Big{)} dx=\int_{\R^{p}}K(y)dy=\int_{\R^{p}}K.$$
Again for statement 2, let $y=x/\e$ and choose $\delta>0.$ Now
\begin{align*}
\int_{||x||>\delta}|K_{\e}(x)|dx & =\frac{1}{\e^{p}}\int_{||x||>\delta}\Big{|}K\Big{(}\frac{x}{\e}\Big{)}\Big{|}dx \\
& =\int_{||y||>\delta/\e}|K(y)|dy.
\end{align*}
Since $\int_{\R^{p}}|K(t)|dt<\infty$ and $\delta/\e \to \infty$ as $\e \to 0$, the proof is complete.
\end{proof}

\begin{theorem}
\label{apprthm}
Let $u_{\e} := u*K_{\e}$, where $\int_{\R^p}K = 1, ~ \|u\|_{1}<\infty$ and $K(x) = o(\|x\|^{-p})$ as $\|x\| \to \infty $. Then $u_{\e} \to u$ as $\e \to 0$ at each point of continuity of $u.$
\end{theorem}
\begin{proof}
Let $u$ be continuous at $x.$ By the definition of continuity, given $\eta >0,$ there exists $\delta >0$ such that $|u(x-t)-u(x)|<\eta$ if $||t||<\delta.$ Since 
$\int_{\R^{p}} K_{\e}=\int_{\R^{p}} K=1 $ from Lemma \ref{kerlem} and hypothesis we have,
\begin{align*}
&|u_\e(x)-u(x)|=\Bigg{|}\int u(x-t)K_{\e}(t)dt-u(x)\int K_{\e}(t)dt \Bigg{|}\\
&= \Bigg{|}\int\displaylimits_{||t||<\delta} \big{(}u(x-t)-u(x)\big{)}K_{\e}(t)dt \\
&\hspace{3.45cm}+\int\displaylimits_{||t||\geq\delta} \big{(}u(x-t)-u(x)\big{)}K_{\e}(t)dt \Bigg{|}\\
&\leq \eta \int\displaylimits_{||t||<\delta}|K_{\e}(t)|dt + \int\displaylimits_{||t||\geq\delta}|u(x-t)-u(x)||K_{\e}(t)|dt \\
&\leq \eta \|K\|_1+\int\displaylimits_{||t||\geq\delta}|u(x-t)||K_{\e}(t)|dt\\ 
& \hspace{5cm} +|u(x)|\int\displaylimits_{||t||\geq\delta}|K_{\e}(t)|dt.
\end{align*}
The third term approaches zero with $\e$ by Lemma \ref{kerlem}. It is enough to show that the second term approaches zero. 
From the hypothesis $|K(x)|=\mu(x)||x||^{-p}$ for some non-negative $\mu(x)$ where $\mu(x) \to 0$ as $||x|| \to \infty.$ We have
\begin{align*}
&\int\displaylimits_{||t||\geq\delta}|u(x-t)||K_{\e}(t)dt| \\ 
= & \int\displaylimits_{||t||\geq\delta}|u(x-t)|\mu\Big{(}\frac{t}{\e}\Big{)}||t||^{-p}dt\\
\leq &\delta^{-p}\Big{\{}\sup_{||t||\geq\delta}\mu\Big{(}\frac{t}{\e}\Big{)}\Big{\}}\int\displaylimits_{||t||\geq\delta}|u(x-t)|dt\\
\leq & \delta^{-p}\Big{\{}\sup_{||t||\geq\delta}\mu\Big{(}\frac{t}{\e}\Big{)}\Big{\}}\|u\|_1.
\end{align*}
Again from the hypothesis $\mu\big{(}\frac{t}{\e}\big{)} \to 0$ as $|\frac{t}{\e}| \to \infty.$ So $\sup\displaylimits_{||t||\geq\delta}\mu\big{(}\frac{t}{\e}\big{)} \to 0$ as $\e \to 0$. This concludes the proof. This theorem is utilized to obtain an approximate analytical form for the gradient of density (refer Section \ref{cov_analysis}). 
\end{proof}

\begin{corollary}
\label{apprcor}
Let $\g u_{\e}=\g u*K_{\e}$, where $\int_{\R^p}K = 1, ~ \|\g u\|_{1}<\infty$ and $K(x) = o(\|x\|^{-p})$ as $\|x\| \to \infty $. Then $\g u_{\e} \to \g u$ as $\e \to 0$ at each point of continuity of $\g u.$ Here the convolution, $\g u*K_{\e}$, is performed component wise.
\end{corollary}
\begin{proof}
The result is obtained by applying Theorem \ref{apprthm} to each component of $\g u$.
\end{proof}
\subsection{Stochastic Gradient Ascent}
Stochastic gradient ascent \cite{borkar2008stochastic} deals with the study of iterative algorithms of the type 
\begin{equation}
\label{algo:optm}
x_{k+1}=x_{k}+a_{k}[\g h(x_{k})+ N_{k+1}]. 
\end{equation}
Here $x_k \in \R^p, ~ k \geq 0$ are the parameters that are updated according to \eqref{algo:optm}. The function $h: \R^p \rightarrow \R$ is an underlying cost function whose maximum we are interested in finding. Also, $a_k, k \geq 0$ is a prescribed step-size sequence. Further, $N_{k+1}, k \geq 0$ constitute the noise terms. 
We state here a theorem that is utilized in the convergence analysis of our algorithm. Consider the following assumptions \cite{borkar2008stochastic, bhatnagar2012stochastic}.
\renewcommand{\labelenumi}{\textbf{A\arabic{enumi}.}}
\begin{enumerate}
    \item The step-sizes $a_k$, $k\geq 0$ satisfy the requirements: $$a_k > 0 ~ \forall k, \sum_k a_k = \infty, \sum_k a_k^2 < \infty.$$
    \item The sequence $N_k, k \geq 0 $ is a martingale difference sequence with respect to the following increasing sequence of sigma fields:
    $$\F_k:=\sigma\{x_0,N_1,\cdots,N_k\} , k\geq 0.$$
    Thus, in particular, $\forall k \geq 0$, $$\E[N_{k+1}|\F_k]=0 \text{ a.s}.$$ Further $N_k, k \geq 0$ are square integrable and
    $$\E[\|N_{k+1}\|^2|\F_k]\leq C(1+\|x_k\|^2) \text{ a.s.}$$ for a given constant $C >0.$
    \item The function $\g h: \R^p \to \R^p$ is Lipschitz continuous.
    \item The functions $\g h_c(x) := \frac{\g h(cx)}{c}, ~ c \geq 1, x \in \R^p$, satisfy $\g h_{c}(x) \rightarrow \g h_{\infty}(x)$ as $c \rightarrow \infty$, uniformly on compacts.
    Furthermore, the o.d.e 
    \begin{equation}
    \dot{x}(t)=\g h_{\infty}(x(t))
    \end{equation}
  \end{enumerate} has the origin as the unique globally asymptotically stable equilibrium. 
   
Consider the ordinary differential equation
\begin{equation}
\label{star}
\dot{x}(t)=\g h(x(t)).
\end{equation}
 Let $H$ denote the compact set of asymptotically stable equilibrium points of the ODE \eqref{star}.

\begin{theorem}
\label{bmt}
Under (A1)-(A4), $\sup_n\|x_n\|< \infty$ (stability) a.s. Further $x_k \rightarrow H$ almost surely as $k \rightarrow \infty$. 
\end{theorem}
\begin{proof}
Follows as a consequence of Theorem 2 in chapter 2 and Theorem 7 in chapter 3 of \cite{borkar2008stochastic}.
\end{proof}

 \section{Motivation and Algorithm}
In this section we motivate and present our iterative algorithm for estimating the mode of a unimodal density. 
The idea of computing the mode is described below. Let $f$ denote the unimodal density function. As mode is the maximizer of the density function, we can estimate the mode by gradient ascent as follows:
\begin{align}\label{update1}
m_{n+1}=m_{n}+a_{n}\g f(m_n),
\end{align}

where $a_{n}$ and $m_{n}$ are the step-size and  current mode estimate, respectively, at time $n$. 

We introduce a function $g$ defined as follows:

\begin{align}\label{formulation2}
g(m) = f(m) - \frac{1}{2} \lambda \|m\|^{2},
\end{align}

where $\lambda >0$ is the regularization coefficient \cite{tikhonov2013numerical}. The idea is to find an $m$ that maximizes the function $g(m)$. This is done to maintain the stability of the estimates in our algorithm (refer proposition 3 in Section IV). Therefore the gradient ascent update is performed as follows:

\begin{align}\label{update2}
m_{n+1} &=m_{n}+a_{n}\g g(m_n) \\
&= m_{n}+a_{n} (\g f(m_n) - \lambda m_n) \label{modupeq}.
\end{align}
It remains to be shown that solution obtained using this update equation \eqref{modupeq} converges to the mode obtained using the update equation \eqref{update1} as $\lambda \xrightarrow{} 0$.
Let $$\hat{m}(\lambda):= \arg \max_{m} g(m) $$ and $$m^{*}:=\arg\max_{m} f(m).$$
It is easy to see that $$\hat{m}(0)=m^{*}.$$
From the continuity of $\arg\max (.)$ function given by the Maximum Theorem \cite{berge1997topological} we have as $\lambda \rightarrow 0$,  $$\hat{m}(\lambda)
\rightarrow \hat{m}(0)=m^{*}.$$ 

The update equation \eqref{modupeq}, however, needs the information of $\g f(m_n)$, which is not known. We therefore make use of the ideas in section II to estimate $\g f(m)$ as follows. To make the notations easy, we replace $m_n$ with $m$ and derive $\g f(m).$   
Applying Corollary \ref{apprcor} to $\g f$ with the kernel $K$ we get for small $\e >0$,
\begin{equation}
\label{eq:approx1}
\g f(m) \approx \g f_{\e}(m)=\int_{\R^{p}}\g f(m-t) K_{\e}(t)dt.
\end{equation}
By the properties of convolution
\begin{align}
\int_{\R^{p}}\g f(m-t) K_{\e}(t)dt & =\int_{\R^{p}}\g K_{\e}(m-t)f(t)dt \nonumber \\
&=\E_{X\sim f}[\g K_{\e}(m-X)] \label{eq:approx2}.
\end{align}
Note that there are several valid choices for function $K$ (also called kernel) to obtain approximation of identity. 

Now, \eqref{modupeq} can be re-written using stochastic gradient ascent as follows:
\begin{align} \label{final_update}
m_{n+1} = m_{n} + a_n(\g K_\epsilon (m_n - X_{n+1})-\lambda m_{n}),
\end{align}
where $X_{n+1}$ is the sample obtained at time $n+1$. 




In the following table, we indicate some of the popular kernels \cite{wheeden2015measure}.

\begin{table}[h!]
\begin{center}
\begin{tabular}{l|l}
\ \\
\textbf{Name of the Kernel} & \textbf{Analytical Form}  \\ 
\\ \hline
Gaussian                   & $\frac{1}{\sqrt{2\pi} }e^{-x^2/2}$ \\
\\ \hline
Cauchy                      & $\frac{1}{\pi(1+x^2)}$   \\                     \\ \hline
Fejer                       & $\frac{\sin^2(x)}{\pi x^2}$
        \\ 
        \\ \hline
        \\
Multivariate Gaussian  & $\frac{1}{{(2\pi)}^{n/2}}\e^{\frac{-x^Tx}{2}}$
\\
\\ 
\end{tabular}
\caption{Examples of kernels}
\label{kernel}
\end{center}
\end{table}

It is easily verified that the kernels defined in Table \ref{kernel} satisfy the hypotheses of Corollary \ref{apprcor}. 
The full algorithm for estimating the mode from  online streaming data is described in Algorithm 1.

\begin{algorithm}
\caption{Calculation of the mode}\label{alg:mode}
\hspace*{\algorithmicindent} \textbf{Input:} $\lambda$: A small regularization coefficient. \\
\hspace*{\algorithmicindent} $X_{n+1}:$ Sample input at time $n+1$. \\
\hspace*{\algorithmicindent} $m_{n}$ : Current estimate of mode \\
\hspace*{\algorithmicindent} $a_{n}$ : Step size sequence \\
\hspace*{\algorithmicindent} \textbf{Output:} mode $m_{n+1}$ estimated from samples  \\ \hspace*{\algorithmicindent} $X_1,X_2,\cdots, X_{n+1}$ 
\begin{algorithmic}[1]
\Procedure{Mode}{$X_{n+1},m_{n},a_{n}$}

\State $d_{n+1} = \nabla K_{\epsilon} (m_{n} - X_{n+1})-\lambda m_n $
\State $m_{n+1} = m_{n} + a_{n}d_{n+1}$
\State \textbf{return} $m_{n+1}$ 
\EndProcedure
\end{algorithmic}
\end{algorithm}

Let $m_0$ denote an initial mode estimate and  $\e$, a small constant. The algorithm works as follows. At time $n+1$, the algorithm takes as input the current mode estimate $m_{n}$ and the sample $X_{n+1}$. It then computes the direction $d_{n+1}$ and updates current approximation of the mode $m_n$ along $d_{n+1}$ as shown in step 2. The output of the algorithm is the updated mode estimate computed from samples obtained so far, i.e., $X_1,\ldots,X_{n+1}$. We prove the convergence of the algorithm in the next section.
\section{Convergence Analysis} \label{cov_analysis}
Let $\F_k=\sigma(m_j, 0\leq j\leq k ;X_j, 0< j\leq k), k \geq 0$ be a sequence of sigma fields. Observe that $\{\F_k\}$ forms a filtration. Let $d_{k+1}=\g K_{\e}(m_k-X_{k+1})- \lambda m_{k}.$ Note that $d_k$ is $\F_k$-measurable. Moreover $d_{k}$ is integrable i.e. $\E[\|d_{k}\|]<\infty$ under the assumption that $\g K_{\e} $ is integrable.
Now the basic algorithm can be written as
\begin{equation}
\label{algo:update}
m_{k+1}=m_{k}+a_{k}d_{k+1}=m_{k}+a_{k}(\E[d_{k+1}|\F_k]+N_{k+1}),
\end{equation}
where $N_{k+1}=d_{k+1}-\E[d_{k+1}|\F_k]$ is a mean zero term. Also, $\{N_{k+1}\}$ constitutes a martingale difference sequence (see proof of Proposition 1). Here $\E[.]$ is the expectation with respect to the density $f.$

Our convergence analysis rests on Theorem \ref{bmt}. Our algorithm is in the form of the general iterative scheme \eqref{algo:optm} with
$\g h(m_k)= \E[d_{k+1}|\F_k]= \g f_{\e}(m_k)-\lambda m_k$ and $N_{k+1}=d_{k+1}-\E[d_{k+1}|\F_k]$. We choose Gaussian kernel for our analysis of the algorithm. Similar analysis can be carried out for other choice of kernels. We first validate the assumptions of Theorem \ref{bmt} below. 

The choice $a_k=1/k, ~ k \geq 1$ assures assumption A1. The following proposition validates assumption A2.
\begin{proposition}
$(N_k,\F_k),k\geq 0$ is a martingale difference sequence with
$$\E[\|N_{k+1}\|^2|\F_k]\leq C(1+\|m_k\|^2),$$
for all $k\geq 0$ and for some $C>0$.
\end{proposition}
\begin{proof}
It is easy to see that $\E[N_{k+1}|\F_k]=\E[d_{k+1}-\E[d_{k+1}|\F_k]|\F_k]=0.$ From the foregoing, $N_k$ is $\F_k-$measurable and integrable $\forall k \geq 0$. So clearly $(N_k,\F_k)$ is a martingale difference sequence.
Now 
\begin{align*}
&\E[\|N_{k+1}\|^2|\F_k]\\
\leq & 2\Big{(}\E[\|d_{k+1}\|^2|\F_k]+\E\big{[}\|\E[d_{k+1}|\F_k]\|^2|\F_k\big{]}\Big{)}\\
\leq & 4\E[\|d_{k+1}\|^2|\F_k].
\end{align*}
The first inequality follows from the simple identity $(a-b)^{2}\leq 2(a^2+b^2)$, while the second inequality follows from a simple application of Jensen's inequality. Since the higher derivatives and in particular Hessian of $K_{\e}$ is bounded, it follows that $\g K_{\e}$ is Lipschitz continuous. We have 
\begin{align*}
\|\g K_{\e}(m)\|-\|\g K_{\e}(0)\|\leq\|\g K_{\e}(m)-\g K_{\e}(0)\|\leq L\|m\|,
\end{align*}
where $L > 0$ is the Lipschitz constant.
Hence for all $m$, 
\begin{align}
\label{eq:Lipschitz}
\|\g K_{\e}(m)\|\leq C_0(1+\|m\|),
\end{align} where $C_0=\max\{\|\g K_{\e}(0)\|,L\}.$
Therefore,
\begin{align*}
& \hspace{0.5cm} \E[\|N_{k+1}\|^2|\F_k] \\
&\leq 4\E[\|d_{k+1}\|^2|\F_k] \\
&\leq 8\E[\|\g K_{\e}(m_k-X_{k+1})\|^2+\|\lambda m_k\|^2|\F_k]\\
&\leq 8\E[C^2_0(1+\|(m_k-X_{k+1})\|)^2+\lambda ^2\|m_k\|^2|\F_k]\\
&\leq 8\E[(2C^2_0+(4C^2_0+\lambda^2)\|m_k\|^2+4C^2_0\|X_{k+1}\|^2)|\F_k]\\
&= 8(2C^2_0+(4C^2_0+\lambda^2)\|m_k\|^2+4C^2_0C_1)\\
&= C(1+\|m_k\|^2),
\end{align*}
where $C_1=\E[\|X_{k+1}\|^2]$ and $C=8\max\{2C_0^2+4C_0^2C_1,4C_0^2+\lambda^2\}.$
This completes the proof.
\end{proof}
The following lemma is useful in proving assumption A3 (see Proposition 2).
\begin{lemma}
\label{lemma2}
$\E[d_{k+1}|\F_k]=\g f_{\e}(m_k)-\lambda m_k.$
\end{lemma}
\begin{proof}
Now $\E[d_{k+1}|\F_k]=\E[\g K_{\e}(m_{k}-X_{k+1})|\F_{k}]- \lambda m_k.$
Also $\g K_{\e} $ is analytic and has a power series expansion around $m_{k}.$ Using power series of $\g K_{\e}$, linearity of the expectation and independence of $X_{k+1}$ from $\F_{k}$ we obtain $\E[\g K_{\e}(m_{k}-X_{k+1})|\F_{k}]=\E[\g K_{\e}(m_{k}-X)]=\g f_{\e}(m_k).$
\end{proof}
Owing to Lemma \ref{lemma2} our iterative update \eqref{algo:update} transforms into 
$m_{k+1}=m_{k}+a_{k}(\g f_{\e}(m_k)-\lambda m_k +N_{k+1})$ and we validate assumption A3 below.
\begin{proposition}
$\g f_{\e}(m) - \lambda m$ is Lipschitz continuous.
\end{proposition}
\begin{proof}
Now for any $x$ and $y$
\begin{align*}
&\|\g f_{\e}(x)-\g f_{\e}(y)-\lambda(x-y)\|\\
\leq &\|\E[\g K_{\e}(x-T)]-\E[\g K_{\e}(y-T)]\|+ \lambda\|x-y\|\\
\leq & \E\|[\g K_{\e}(x-T)]-\g K_{\e}(y-T)\| + \lambda\|x-y\|\\
\leq & (L+\lambda)\|x-y\|,
\end{align*}
where $L$ is the Lipschitz constant of $\g K_{\e}.$
\end{proof}
The following proposition proves assumption A4.
\begin{proposition} The ODE $\dot{m}=h_{\infty}(m)$ has the origin as its unique globally asymptotically stable equilibrium point.
\end{proposition}
\begin{proof}
From the definition of $h_{\infty}(m)$, see assumption A4, we have
\begin{align*}
h_{\infty}(m) &= \lim\displaystyle_{c \to \infty} \frac{\g f_{\e}(cm)-\lambda cm}{c}\\
&=\lim\displaystyle_{c \to \infty}\frac{\E[\g K_{\epsilon}(cm-X)]}{c}-\lambda m\\
&=\lim\displaystyle_{c \to \infty}\frac{1}{c}\int_{\R^{p}} \frac{2(x-cm)}{\e^3\sqrt{\pi }}e^{\frac{-\|cm-x\|^2}{\e^2}}f(x)dx-\lambda m\\
&=\lim\displaystyle_{c \to \infty} \frac{1}{c} \int_{\R^{p}} \frac{2x}{\e^3\sqrt{\pi}}e^{\frac{-\|cm-x\|^2}{\e^2}}f(x)dx \\
& \hspace{1cm} -\lim\displaystyle_{c \to \infty}\int_{\R^{p}} \frac{2m}{\e^3\sqrt{\pi }}e^{\frac{-\|cm-x\|^2}{\e^2}}f(x)dx - \lambda m\\
&=-\lambda m.
\end{align*}
Here
\begin{align*}
& \Bigg{ \|}\frac{1}{c} \int_{\R^{p}} \frac{2x}{\e^3\sqrt{\pi}}e^{\frac{-\|cm-x\|^2}{\e^2}}f(x)dx \Bigg{\|} \\
&\leq \frac{1}{c} \int_{\R^{p}}   \Big{\|}\frac{2x}{\e^3\sqrt{\pi}}e^{\frac{-\|cm-x\|^2}{\e^2}}f(x)dx \Big{\|} \\
&\leq \frac{1}{c} \int_{\R^{p}} \Big{\|} \frac{2x}{\e^3\sqrt{\pi}}f(x)dx \Big{\|} \\
& \rightarrow 0 \text{ as } c \rightarrow \infty,
\end{align*}
where the facts that $e^{\frac{-\|cm-x\|^2}{\e^2}} \leq 1$ and 
$\int_{\R^{p}} \|xf(x)dx\| < \infty$ are utilized.
By the application of Dominated Convergence Theorem \cite{wheeden2015measure} we have
\begin{align*}
& \frac{2m}{\e^3\sqrt{\pi }} \lim\displaystyle_{c \to \infty}\int_{\R^{p}} e^{\frac{-\|cm-x\|^2}{\e^2}}f(x)dx \\
& = \frac{2m}{\e^3\sqrt{\pi }} \int_{\R^{p}} \lim\displaystyle_{c \to \infty}e^{\frac{-\|cm-x\|^2}{\e^2}}f(x)dx =0.
\end{align*}
So we have that $h_{\infty}(m)=-\lambda m.$ \\
Now for the system $\dot{m}=h_{\infty}(m)=-\lambda m$, clearly the origin is an equilibrium point. Also for any initial point $m_0$, $m(t)=m_0\exp(-\lambda t)$ is the solution of the system and $m(t) \rightarrow 0$ as $t \rightarrow \infty$. Therefore the origin is the unique globally asymptotically stable equilibrium point of the system. This concludes the proof.
\end{proof}
\begin{remark}
Note that the regularization coefficient $\lambda$ plays a key role in establishing the stability of the mode estimates. To see the effect of the regularization term consider the iterates 
$$m_{n+1}=m_{n}+a_{n}\g f_{\e}(m_n).$$ The $h_{\infty}(.)$ corresponding to this update equation is identically 0 thereby violating assumption A4.
\end{remark}
Consider now the following ODE:
\begin{align} \label{final}
    \dot{x} = \nabla f_{\epsilon}(x(t)) - \lambda x.
\end{align}
Let $\bar{H}$ be the set of asymptotically stable equilibrium points of \eqref{final}. 
\begin{remark}
From our assumptions $\g f_\e (x) \xrightarrow{} \g f(x)$ as $\epsilon \xrightarrow{} 0$ for every point of continuity $x$ and $\bar{H} \rightarrow \{m^{*}\}$ as $\lambda \rightarrow 0$.
\end{remark}
We have the following as our main result.

\begin{theorem}
${m_n}, ~ n \geq 0$ obtained from Algorithm 1 satisfies $m_n \xrightarrow[]{} \bar{H}$ a.s.
\end{theorem}

\begin{proof}
The result follows from the foregoing and Theorem 2.
\end{proof}

\section{Experiments}

In this section, we discuss the numerical performance of our algorithm. 
We implement our algorithm on known popular distributions. We collect $10^6$ samples from a known distribution and apply our algorithm for estimating the mode. The initial mode estimate is selected as the average of initial 1000 points. We consider Gaussian kernel for univariate distributions and multivariate Gaussian kernel for bivariate normal and Dirchlet distributions (see Table \ref{kernel}) for our experiments. The regularization coefficient $\lambda$ is chosen to be $10^{-5}$ and $\epsilon$ is set to 1. We perform 100 runs of the experiment and estimated mode is calculated as the mean of modes obtained over the 100 runs. The following Table \ref{table1} illustrates the performance of our algorithm (estimated mode) on standard distributions. We have also indicated the actual mode of the distribution in the Table \ref{table1}. The code for our experiments can be found at \url{https://github.com/raghudiddigi/Mode-Estimation}. 

\begin{table}[H]
\centering
\begin{tabular}{|l|l|l|}
\hline
\textbf{Distribution} & \textbf{Actual Mode} & \textbf{Estimated Mode $\pm$ Std.Dev} \\ \hline
Normal                & 10                   & 9.971783 $\pm$ 0.333930                 \\ \hline
Gamma                 & 5                    & 5.182626 $\pm$ 0.306337                 \\ \hline
Exponential           & 0                    & 0.697886 $\pm$ 0.007722             \\ \hline
Weibull               & 0                    & 0.697192 $\pm$ 0.007544                \\ \hline
Beta                  & 1                    & 0.900645 $\pm$ 0.001356                 \\ \hline
Bivariate Normal                 & [20; 15]                    & [20.030044; 15.015614 ] $\pm$ [0;0]                 \\ \hline
Dirichlet               & [0.5; 0.5]                    & [0.498404; 0.501579] $\pm$ [0;0]                 \\ \hline
\end{tabular}
\caption{Performance of our proposed algorithm on standard distributions}
\label{table1}
\end{table}
It is interesting to note that, though Exponential and Weibull densities are not smooth and do not satisfy our assumptions, the estimated mode obtained by our algorithm is closer to the actual mode.

In Figure \ref{Initial}, we show the performance of our algorithm with different initial points. For this purpose, we select Normal distribution with mean 10. We implement our algorithm with initial points 5,10 and 15 and plot the estimated mode over initial 50,000 iterations. We observe that the estimates of the mode in all the three cases converge towards the actual mode having value 10 as the number of iterations increase. This shows that the proposed algorithm is not very sensitive with respect to the initial mode estimate. These results thus confirm the practical utility of our algorithm. 
\begin{figure}[ht]
    \centering
    \includegraphics[scale = 0.26]{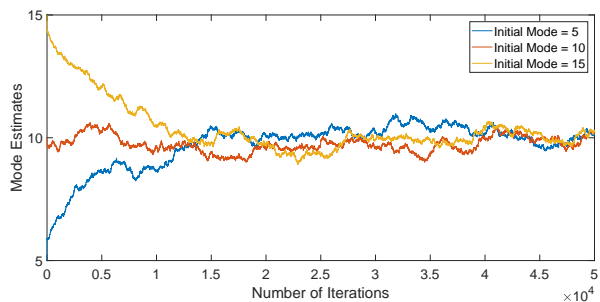}
    \caption{Performance of proposed algorithm with different initial points}
    \label{Initial}
\end{figure}

\section{Conclusions and Future Work}

In this paper, we proposed an online computationally efficient algorithm for computing the mode from the samples of an unknown  density. We have provided the proofs for the stability of the iterates and convergence of our algorithm. Next, we showed results of experiments on standard distributions that demonstrate the effectiveness of our algorithm in practice.


In future, we wish to propose second order algorithms based on the Newton's method in the place of gradient ascent. Newton's method is known to converge faster than the gradient ascent method.
Another interesting future direction would be to obtain finite sample error bounds and rate of convergence for our algorithm by utilizing central limit theorem for stochastic approximation.

\bibliographystyle{IEEEtran}
\bibliography{main}

\end{document}